\documentclass[sigconf]{acmart}

\usepackage{amsmath,amssymb,amsfonts}
\usepackage{bm}
\usepackage{algorithm}
\usepackage{algpseudocode}
\usepackage{booktabs}
\usepackage{balance}

\newcommand{\E}{\mathbb{E}}

\newcommand{\given}{\,|\,}

\newtheorem*{theorem*}{Theorem}
\AtBeginDocument{%
  }

\setcopyright{cc}
\setcctype{by}
\copyrightyear{2026}
\acmYear{2026}
\acmDOI{10.1145/3770855.3817847}
\acmConference[KDD 2026]{Proceedings of the 32nd ACM SIGKDD Conference on Knowledge Discovery and Data Mining V.2}{August 9--13, 2026}{Jeju Island, Republic of Korea.}
\acmBooktitle{Proceedings of the 32nd ACM SIGKDD Conference on Knowledge Discovery and Data Mining V.2 (KDD 2026), August 9--13, 2026, Jeju Island, Republic of Korea}
\acmISBN{979-8-4007-2259-2/2026/08}
\begin{document}

\title{Constrained Auto-Bidding via Generative Response Modeling}

\author{Eunseok Yang}
\email{eunseok.y@navercorp.com}
\affiliation{%
  \institution{NAVER Corporation}
  \city{Seongnam-si}
  \country{Republic of Korea}
}

\author{Xingdong Zuo}
\email{xingdong.zuo@navercorp.com}
\affiliation{%
  \institution{NAVER Corporation}
  \city{Seongnam-si}
  \country{Republic of Korea}
}

\author{Kyung-Min Kim}
\email{kyungmin.kim.ml@navercorp.com}
\affiliation{%
  \institution{NAVER Corporation}
  \city{Seongnam-si}
  \country{Republic of Korea}
}

\renewcommand{\shortauthors}{Yang et al.}

\begin{abstract}
Auto-bidding systems aim to maximize advertiser value over long horizons 
under budget constraints and ratio targets such as cost-per-acquisition, 
yet future traffic and auction dynamics are non-stationary and uncertain. 
Existing approaches face distinct limitations:
control-based pacing reacts to deviations but cannot anticipate future conditions,
while RL and generative methods fold constraints into reward signals,
obscuring violations and degrading under distribution shift.
We shift the learning target from \emph{actions} to \emph{responses} 
with the \emph{Generative Response Model} (GRM), a history-conditioned 
sequence model that jointly predicts future traffic volume and 
horizon-aggregate cost/value curves as functions of a single bid multiplier. 
We show that under mild monotonicity conditions, the optimality gap 
relative to full per-tick control is bounded by the dispersion of 
per-tick marginal value-per-cost. 
Given predicted responses, a lightweight analytic controller enforces 
each active constraint via a 1D root-finding step. 
We prove this controller is exact for the single-multiplier problem 
and bound constraint violations under receding-horizon replanning 
in terms of prediction error. 
Experiments on AuctionNet show that GRM improves constraint stability 
and overall score compared to existing baselines.
\end{abstract}



\maketitle

\section{Introduction}
Online advertising platforms generate billions of auction opportunities daily, where advertisers compete in real time for impressions \cite{edelman2007internet,yuan2013real}.
Auto-bidding systems maximize advertiser value (e.g., conversions or revenue) under long-horizon constraints: a fixed campaign budget together with an efficiency target such as cost-per-acquisition (CPA) or return-on-ad-spend (ROAS).
Constraints span the full horizon, while bidding decisions are made impression by impression under uncertain future traffic and competition.

Production systems commonly decompose this into two layers \cite{he2021unified,aggarwal2024auto}: a per-impression value estimator, and a campaign-level pacing multiplier $\alpha$ that scales values into bids.
The control problem is then to update $\alpha$ over time so that cumulative spend and KPI ratios stay feasible at the horizon, despite uncertain remaining supply and conversion efficiency.

Prior approaches fall into two families.
The first outputs bids or pacing multipliers directly.
Control-based pacing updates $\alpha$ from feedback on realized spend and efficiency \cite{zhang2016feedback,he2021unified}, yielding a stable controller that cannot anticipate future conditions.
Primal-dual and online-learning methods give no-regret bidding under budget and return-on-spend constraints with formal guarantees \cite{yu2017online,zhang2022leveraging}, but rely on stylized stochastic assumptions and do not learn history-conditioned responses.
Reinforcement learning approaches learn bidding policies from offline data via value estimation or policy regularization \cite{wu2018budget,cai2017real}, but typically encode constraints through reward shaping, which is brittle under distribution shift.
Generative sequence models such as Decision Transformer \cite{chen2021decision} capture long histories by generating actions conditioned on return-to-go (RTG), with constraints folded into the RTG signal or enforced by post-hoc search. Feasibility is mediated through the conditioning value, so the model exposes no direct mechanism to bound or diagnose violations. The RTG itself must be specified at inference time, which adds uncertainty over long horizons.

The second family predicts environment quantities for a separate controller to consume.
Bid-landscape models predict per-auction win-rate or price distributions as functions of the bid \cite{cui2011bid,ghosh2019scalable,wang2016functional}, but treat each auction as a snapshot and need separate layers to enforce horizon-level constraints.
Forecast-based controllers such as MPC-style pacing plan with predicted spend-rate dynamics, but address cost or supply alone.
Neither family directly predicts how cost and value respond to the bid multiplier, so constraints are either folded into the objective or enforced by a separate layer.

We propose a different decomposition: learn the environment's response rather than the optimal action.
Our \textbf{Generative Response Model (GRM)} is a history-conditioned model that predicts a horizon response bundle: the horizon-aggregate expected cost curve $\bar{C}(\alpha)$, the horizon-aggregate expected value curve $\bar{V}(\alpha)$, and the remaining traffic over the horizon.
Given this predicted response, a lightweight analytic controller computes the budget-feasible multiplier $\alpha_B$ and the efficiency-feasible multiplier $\alpha_C$ via two one-dimensional root solves, and executes
\[
\alpha_t = \min\{\alpha_B,\alpha_C\}.
\]
Each plan treats $\alpha$ as a single value for the remainder of the horizon, but the controller re-solves at every tick on updated state, so the executed sequence $\alpha_1,\ldots,\alpha_T$ varies tick to tick.
Because the controller solves a feasibility problem on predicted curves, constraint handling is explicit and violations are tied to specific prediction errors.
Under receding-horizon replanning, constraint violations scale linearly with prediction error (Theorem~\ref{thm:violation}), and better prediction directly improves constraint safety.
Response curves are monotone and smooth in $\alpha$, whereas optimal actions change discontinuously at constraint boundaries. This makes the curves easier targets for supervised learning.

We further simplify by predicting one horizon-aggregate curve rather than a separate curve per tick, which incurs a gap controlled by efficiency dispersion, the traffic-weighted variance of per-tick marginal value-per-cost (Theorem~\ref{thm:gap}). When dispersion is small, this approximation is near-optimal and gives a lower-dimensional, more stable prediction target.

\smallskip\noindent\textbf{Contributions.}
We introduce the \emph{Generative Response Model} (GRM, Figure~\ref{fig:overview}), a history-conditioned model that predicts horizon-aggregate response curves (cost, value, traffic) rather than actions.
An analytic min-pacing controller then computes constraint-feasible multipliers via two 1D root solves and recovers the exact single-$\alpha$ solution (Theorem~\ref{thm:exact}).
We prove that the single-$\alpha$ gap is bounded by efficiency dispersion (Theorem~\ref{thm:gap}) and that constraint violations scale with prediction error (Theorem~\ref{thm:violation}).
On AuctionNet \cite{su2024a}, GRM outperforms the strongest baseline by $7.8\%$ and degrades less under two distribution-shift scenarios.

\begin{figure*}[t]
  \centering
  \includegraphics[width=\linewidth]{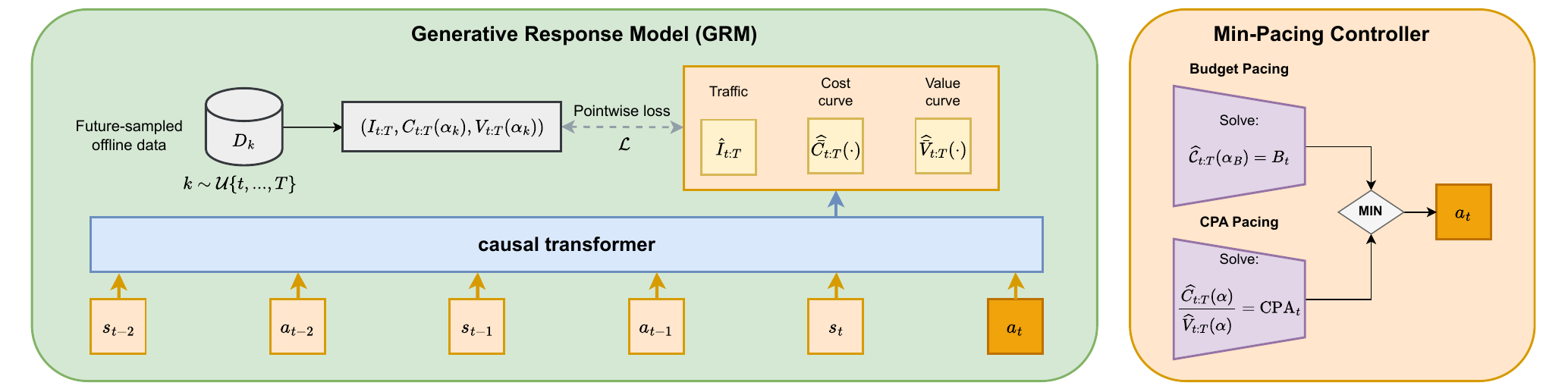}
  \caption{Overview of the Generative Response Model (GRM) framework.
    (Left) GRM encodes state-action history via a causal transformer and predicts
    horizon-aggregate response curves (traffic $\hat{I}_{t:T}$, cost $\widehat{\bar{C}}_{t:T}(\cdot)$, value $\widehat{\bar{V}}_{t:T}(\cdot)$).
    Training uses future-sampled supervision, where for each anchor tick $t$ we sample future ticks $k\sim\mathcal{U}\{t,\dots,T\}$
    from offline data and fit curves to realized outcomes.
    (Right) The min-pacing controller solves two 1D root equations for budget and CPA constraints,
    then executes $\alpha_t = \min\{\alpha_B, \alpha_C\}$.}
  \label{fig:overview}
\end{figure*}

\section{Problem Setup}\label{sec:problem-setup}

\paragraph{General bid optimization.}
In real-time bidding, an advertiser faces a stream of impression opportunities over a campaign horizon.
At each opportunity $i$, the advertiser observes features $x_i$ (user attributes, ad slot, context)
and must submit a bid $b_i \ge 0$ before the auction clears.
Let $v_i$ denote the estimated value of winning impression $i$ (e.g., predicted conversion probability
or expected revenue), $c_i(b_i)$ the cost incurred, and $u_i(b_i)$ the realized value.
The general constrained bid optimization problem is:
\begin{align}
\max_{b_1,\dots,b_N}~ & \sum_{i=1}^{N} u_i(b_i) \label{eq:general-obj}\\
\text{s.t.}\quad
& \sum_{i=1}^{N} c_i(b_i) \le B, \label{eq:general-budget}\\
& \frac{\sum_{i=1}^{N} c_i(b_i)}{\sum_{i=1}^{N} u_i(b_i)} \le \tau, \label{eq:general-cpa}
\end{align}
where $B$ is the total budget and $\tau$ is the target CPA or inverse target ROAS.
Three factors make this formulation intractable:
(i) the number of opportunities $N$ can exceed millions per day;
(ii) decisions are sequential under uncertainty about future supply and competition; and
(iii) constraints \eqref{eq:general-budget}--\eqref{eq:general-cpa} couple all decisions across the horizon.

\paragraph{Multiplier-based pacing.}
Production auto-bidding systems address this complexity through a two-layer decomposition.
A per-impression value estimator produces $v_i$, and a campaign-level \emph{pacing multiplier} $\alpha$ scales values into bids via $b_i = \alpha\, v_i$.
This reduces the problem from choosing $N$ individual bids to choosing a single multiplier $\alpha$, while the value estimator is trained separately and held fixed during bid optimization.
Multiplicative pacing of this form preserves the relative ordering of impression values and gives a single control parameter for budget and efficiency \cite{conitzer2022pacing,balseiro2023field}.

\paragraph{Tick-level formulation.}
We further aggregate time into discrete \emph{ticks} $t=1,\dots,T$ (e.g., minutes or hours),
where each tick contains $I_t$ impression opportunities.
At tick $t$, the bidder chooses a multiplier $\alpha_t \in \mathcal{A}$
and applies it uniformly to all impressions within that tick,
where $\mathcal{A}=[\underline\alpha,\bar\alpha]\subseteq \mathbb{R}_{+}$ is a bounded operating range:
\begin{align}
b_{t,i}(\alpha_t) \;=\; \alpha_t\, v_{t,i}, \qquad i=1,\dots,I_t.
\end{align}
Tick-level aggregation matches the coarse time scale of campaign-level constraints and keeps the state space tractable for sequential decision-making, while preserving the essential budget-pacing dynamics \cite{zhang2016feedback,he2021unified}.

\paragraph{Information structure within a tick.}
At the beginning of tick $t$, the bidder observes pre-decision history
\begin{align}
H_t := (s_{1:t},\, \alpha_{1:t-1},\, I_{1:t-1},\, \mathrm{Cost}_{<t},\, \mathrm{Val}_{<t}),
\end{align}
where $s_k$ denotes contextual features (time-of-day, campaign state, etc.).
The bidder then commits to a multiplier $\alpha_t$ before observing the bid opportunities for tick $t$.
Subsequently, $I_t$ opportunities arrive stochastically, auctions are resolved, and outcomes (cost, value) are realized.
At decision time $t$, the future traffic volumes $I_{t:T} := \sum_{k=t}^{T} I_k$ and
response curves are unknown and must be predicted from $H_t$.

\paragraph{Spend and value curves as functions of $\alpha$.}
Let $c_{t,i}(\alpha)$ and $u_{t,i}(\alpha)$ denote the cost and realized value obtained at opportunity $i$ in tick $t$ when bidding with multiplier $\alpha$.
We define per-opportunity expected curves conditioned on the pre-decision history:
\begin{align}
C_t(\alpha) &:= \mathbb{E}\!\left[c_{t,i}(\alpha) \given H_t\right], \\
V_t(\alpha) &:= \mathbb{E}\!\left[u_{t,i}(\alpha) \given H_t\right],
\end{align}
where the expectation is over auction outcomes (competition, winning, conversion) given history $H_t$.
Here $v_{t,i}$ is the per-impression estimated value, used as the bid scaler in $b_{t,i} = \alpha\, v_{t,i}$, while $u_{t,i}$ is the realized KPI outcome (conversion indicator under a CPA target, revenue under a ROAS target).
Both target-CPA (cost/conversions $\le \tau$) and target-ROAS (revenue/cost $\ge \rho$) cases reduce to the same ratio form $\sum C_t / \sum V_t \le \tau$ \cite{aggarwal2024auto,balseiro2023field}. We phrase the remainder of the paper in CPA terms for concreteness, and the same arguments apply to ROAS after reinterpreting $u_{t,i}$ as revenue.

\paragraph{Structural assumptions on response curves.}
We impose mild regularity conditions that are standard in pacing and auction theory:
\begin{itemize}
\item[(S1)] $C_t(\alpha)$ is strictly increasing in $\alpha$, and $V_t(\alpha)$ is non-decreasing;
\item[(S2)] Both curves are bounded: $C_t(\alpha), V_t(\alpha) \in [0, \bar{M}]$ for some $\bar{M} < \infty$.
\end{itemize}
Assumption (S1) reflects that higher multipliers yield more aggressive bidding, increasing both spend and (weakly) value.
Assumption (S2) holds because per-opportunity cost is bounded by the maximum bid and value is bounded by the outcome space.
Under these properties, the constraint equations admit unique root solutions (\S\ref{sec:control}).

\paragraph{Constrained objective.}
The tick-level reformulation of \eqref{eq:general-obj}--\eqref{eq:general-cpa} is:
\begin{align}
\max_{\alpha_{1:T}}~ & \sum_{t=1}^{T} I_t\, V_t(\alpha_t) \label{eq:primal}\\
\text{s.t.}\quad
& \sum_{t=1}^{T} I_t\, C_t(\alpha_t) \le B, \label{eq:budget}\\
& \frac{\sum_{t=1}^{T} I_t\, C_t(\alpha_t)}{\sum_{t=1}^{T} I_t\, V_t(\alpha_t)} \le \tau.
\label{eq:cpa}
\end{align}
Equivalently, \eqref{eq:cpa} becomes $\sum_{t} I_t\, C_t(\alpha_t) \le \tau \sum_{t} I_t\, V_t(\alpha_t)$.
Additional ratio constraints such as ROI floors \cite{su2024spending} fit the same form via extra Lagrange multipliers \cite{aggarwal2024auto,balseiro2023field}.

\section{Generative Response Model (GRM)}
\subsection{Response objects and horizon-aggregate curves}
\paragraph{Horizon-aggregate response curves.}
Recall from \S\ref{sec:problem-setup} that $C_k(\alpha) = \E[c_{k,i}(\alpha) \given H_k]$ and $V_k(\alpha) = \E[u_{k,i}(\alpha) \given H_k]$
are history-conditioned per-opportunity expected cost and value at tick $k$.
Each plan of our controller uses a single multiplier $\alpha$ over the remaining horizon. We therefore define the \emph{horizon-aggregate per-opportunity curves}:
\begin{align}
\bar{C}_{t:T}(\alpha) &:= \frac{\sum_{k=t}^{T} I_k\, C_k(\alpha)}{\sum_{k=t}^{T} I_k}, \\
\bar{V}_{t:T}(\alpha) &:= \frac{\sum_{k=t}^{T} I_k\, V_k(\alpha)}{\sum_{k=t}^{T} I_k}.
\end{align}
These are the traffic-weighted averages of per-opportunity cost and value over the horizon $t{:}T$ under constant $\alpha$.
Because $C_k$ and $V_k$ are conditioned on $H_k$, the aggregate curves $\bar{C}_{t:T}$ and $\bar{V}_{t:T}$
inherit this history-dependence, and GRM learns to predict them from the pre-decision history $H_t$.

\paragraph{Response bundle.}
At decision time $t$, horizon planning requires forecasting future traffic and the aggregate response curves:
\begin{align}
\mathcal{R}_{t:T}
~:=~
\Big(\, I_{t:T},~ \bar{C}_{t:T}(\cdot),~ \bar{V}_{t:T}(\cdot)\Big),
\end{align}
where $I_{t:T}=\sum_{k=t}^{T} I_k$ is the total number of future bid opportunities.
The expected horizon-level totals under constant $\alpha$ are then
$\mathcal{C}_{t:T}(\alpha) = I_{t:T}\cdot \bar{C}_{t:T}(\alpha)$ and
$\mathcal{V}_{t:T}(\alpha) = I_{t:T}\cdot \bar{V}_{t:T}(\alpha)$.

\paragraph{GRM: predicting the horizon-aggregate bundle.}
GRM is a causal sequence model that summarizes pre-decision history into a latent state and predicts
the horizon-aggregate response bundle:
\begin{align}
h_t &= f_{\theta}\!\big(s_{1:t}, \alpha_{1:t-1}\big) \in \mathbb{R}^{d},\\
\widehat{\mathcal{R}}_{t:T} &= g_{\theta}(h_t)
~=~
\Big(\, \widehat{I}_{t:T},~ \widehat{\bar{C}}_{t:T}(\cdot),~ \widehat{\bar{V}}_{t:T}(\cdot)\Big).
\end{align}
Here $s_{1:t}$ denotes the sequence of contextual features up to tick $t$,
and $\alpha_{1:t-1}$ is the executed (or logged) multiplier history.
Note that GRM predicts a single aggregate curve for the horizon, not individual curves for each future tick.

\subsection{Function-valued curve parameterization}
GRM outputs a low-dimensional parameter vector that induces the horizon-aggregate response curves.

\paragraph{Design motivation.}
The parametric form should reflect economic structure while remaining compact.
Cost and value curves are bounded, with no impressions won at $\alpha=0$
and all available impressions won as $\alpha\to\infty$.
This motivates a sigmoid shape with saturation parameter $a$.
Bid-landscape models show that win-rate follows a log-concave pattern
in bid price \cite{cui2011bid,wang2016functional}:
the probability of winning increases rapidly at low bids,
then diminishes as the bidder approaches the maximum clearing price.
Applying the sigmoid to $\log(\alpha)$ rather than $\alpha$ directly
captures this diminishing-returns structure.
We use the normal CDF $\Phi$ for numerical stability.
A minimal parameterization (scale, sensitivity, shift) keeps the prediction target low-dimensional
while guaranteeing monotonicity when $b>0$, satisfying assumption (S1) for Theorems~\ref{thm:exact} and \ref{thm:violation}.

\paragraph{Parametric family.}
We parameterize the aggregate cost curve by a monotone saturating family:
\begin{align}
\widehat{\bar{C}}_{t:T}(\alpha)
~=~ a^{(C)} \cdot \tilde{\Phi}\!\Big( b^{(C)},\, c^{(C)},\, \alpha \Big),
\label{eq:curveC}
\end{align}
where $\tilde{\Phi}(b,c,\alpha) := \frac{\Phi(b\log(\alpha+\varepsilon)+c) - \Phi(b\log\varepsilon+c)}{1 - \Phi(b\log\varepsilon+c)}$
is a normalized sigmoid satisfying $\tilde{\Phi}(b,c,0)=0$ and $\lim_{\alpha\to\infty}\tilde{\Phi}(b,c,\alpha)=1$.
This shift-and-rescale enforces the exact boundary conditions $\widehat{\bar{C}}_{t:T}(0)=0$
and $\widehat{\bar{C}}_{t:T}(\infty)=a^{(C)}$, corresponding to no wins at $\alpha=0$ and full saturation as $\alpha\to\infty$.
Here $\Phi(\cdot)$ is the standard normal CDF, and $\varepsilon>0$ is a small constant (we use $\varepsilon=10^{-3}$).
The value curve $\widehat{\bar{V}}_{t:T}(\alpha)$ is parameterized identically with $\theta^{(V)}=(a^{(V)},b^{(V)},c^{(V)})$.
We enforce $a^{(\cdot)}>0$ and $b^{(\cdot)}>0$ via softplus to guarantee valid, monotone curves.

\paragraph{Predicting curve parameters and traffic.}
GRM outputs the traffic forecast together with the aggregate curve parameters:
\begin{align}
(\widehat{I}_{t:T},\,\theta^{(C)},\,\theta^{(V)})
~=~ g_{\theta}(h_t),
\end{align}
where $\theta^{(C)}=(a^{(C)},b^{(C)},c^{(C)})$. This is a compact $(1+6)$-dimensional output per anchor.

\subsection{Training with future-sampling supervision}
\paragraph{Training data and supervision signal.}
Logged data provide, for each tick $k$, the executed multiplier $\alpha_k$ and tick-level aggregates:
the number of bid opportunities $I_k$, total cost $\mathrm{Cost}_k$, and total realized value $\mathrm{Val}_k$.
We form per-opportunity targets at the logged action:
\begin{align}
C_k(\alpha_k) \;\approx\; \frac{\mathrm{Cost}_k}{I_k},
\qquad
V_k(\alpha_k) \;\approx\; \frac{\mathrm{Val}_k}{I_k}.
\end{align}
Since we only observe outcomes at logged multipliers, we train GRM with \emph{future sampling}:
for an anchor tick $t$, sample future ticks $k \sim \mathcal{U}\{t,\dots,T\}$ and apply point supervision.

\paragraph{Weighted curve fitting.}
GRM predicts a single horizon-aggregate curve $\widehat{\bar{C}}_{t:T}(\alpha)$, yet we supervise using
individual tick observations $(k,\alpha_k,C_k(\alpha_k))$ where each tick has a different logged multiplier $\alpha_k$.
We fit the parametric curve to these scattered data points via traffic-weighted regression, so that high-traffic ticks contribute more to the fit, reflecting their larger share in the aggregate.

\paragraph{Training loss.}
For each anchor tick $t$, we draw $M$ future indices
$k_1,\dots,k_M \stackrel{\text{i.i.d.}}{\sim} \mathcal{U}\{t,\dots,T\}$
and weight each sample by the realized traffic $I_{k_m}$:
\begin{align}
\mathcal{L}(\theta)
&= \E_{t}\bigg[
\frac{1}{M}\sum_{m=1}^{M}
\Big( I_{k_m} \big(\widehat{\bar{C}}_{t:T}(\alpha_{k_m}) - C_{k_m}(\alpha_{k_m})\big)^2 \notag\\
&\qquad\qquad\quad + I_{k_m} \big(\widehat{\bar{V}}_{t:T}(\alpha_{k_m}) - V_{k_m}(\alpha_{k_m})\big)^2 \Big) \notag\\
&\qquad + \lambda_I \big(\log \widehat{I}_{t:T} - \log I_{t:T}\big)^2
\bigg].
\label{eq:loss}
\end{align}
The anchor $t$ defines the prediction window $[t,T]$, and the sample indices $k_m$ select future ticks at which the predicted curve is evaluated against logged observations.
$C_{k_m}(\alpha_{k_m})$ is the observed per-opportunity cost at tick $k_m$, and $\lambda_I>0$ balances the traffic and curve losses.
Traffic weighting by $I_{k_m}$ makes the fitted curve approximate the aggregate $\bar{C}_{t:T}(\alpha)$, which is itself a traffic-weighted average of per-tick curves.
The traffic term uses log-scale to stabilize training under heavy-tailed traffic distributions.

\paragraph{Identifiability and coverage.}
Logged data contain only realized $(\alpha_k, C_k(\alpha_k))$ pairs at distinct histories $H_k \ne H_t$, which raises an identifiability concern for the horizon-aggregate curve.
Our goal, however, is not to recover true causal effects of counterfactual $\alpha$,
but to learn a predictive model $\widehat{\bar{C}}_{t:T}(\alpha \given H_t)$
that generalizes to the deployment distribution.
By conditioning on history $H_t$, which encodes remaining budget, cumulative CPA, and elapsed time,
GRM learns state-dependent response predictions rather than unconditional causal curves.
Production auto-bidding systems naturally generate diverse $\alpha$ trajectories
through pacing controller adjustments, heterogeneous advertisers, and A/B tests,
providing coverage without explicit exploration.

\section{Analytic Constrained Control (Min-Pacing) on Generated Responses}\label{sec:control}
At tick $t$, GRM provides a horizon forecast
$\widehat{\mathcal{R}}_{t:T}=(\widehat{I}_{t:T},\,\widehat{\bar{C}}_{t:T}(\cdot),\,\widehat{\bar{V}}_{t:T}(\cdot))$.
We compute horizon-level total cost and value under a constant multiplier $\alpha$:
\begin{align}
\widehat{\mathcal{C}}_{t:T}(\alpha) &:= \widehat{I}_{t:T}\cdot \widehat{\bar{C}}_{t:T}(\alpha), \\
\widehat{\mathcal{V}}_{t:T}(\alpha) &:= \widehat{I}_{t:T}\cdot \widehat{\bar{V}}_{t:T}(\alpha).
\end{align}
\paragraph{State-dependent remaining constraints.}
Let $\mathrm{Cost}_{<t}$ and $\mathrm{Val}_{<t}$ denote realized cumulative cost and value up to tick $t{-}1$.
Given a total budget $B$ and a target CPA $\tau$, the remaining constraints at tick $t$ are
\begin{align}
B_t &:= B - \mathrm{Cost}_{<t}, \\
\Delta_t &:= \tau\cdot \mathrm{Val}_{<t} - \mathrm{Cost}_{<t}.
\end{align}
Note that both $B_t$ and the CPA slack $\Delta_t$ evolve over time as we consume budget and accrue value.

\paragraph{Two 1D solves: budget pacing and CPA pacing.}
We compute two candidate multipliers via 1D root-finding (e.g., bisection).

\textbf{(i) Budget pacing.} Find $\alpha_B$ such that predicted remaining spend matches remaining budget:
\begin{align}
\widehat{\mathcal{C}}_{t:T}(\alpha_B) ~=~ B_t.
\label{eq:alphaB}
\end{align}

\textbf{(ii) CPA pacing.} Enforcing the overall CPA constraint
$\frac{\mathrm{Cost}_{<t}+\widehat{\mathcal{C}}_{t:T}(\alpha)}
{\mathrm{Val}_{<t}+\widehat{\mathcal{V}}_{t:T}(\alpha)} \le \tau$
is equivalent to the scalar inequality
\begin{align}
\widehat{\mathcal{C}}_{t:T}(\alpha) - \tau\,\widehat{\mathcal{V}}_{t:T}(\alpha) \le \Delta_t.
\end{align}
We define $\alpha_C$ as the boundary (equality) solution:
\begin{align}
\widehat{\mathcal{C}}_{t:T}(\alpha_C) - \tau\,\widehat{\mathcal{V}}_{t:T}(\alpha_C) ~=~ \Delta_t.
\label{eq:alphaC}
\end{align}

\paragraph{Min-pacing control law.}
We execute the conservative multiplier
\begin{align}
\alpha_t ~=~ \min\{\alpha_B,\alpha_C\},
\label{eq:minpacing}
\end{align}
which satisfies both constraints whenever the aggregate curves are monotone in $\alpha$ and the root
solutions exist. In \S\ref{sec:theory} we prove that this min-pacing rule is the exact solution
to the single-$\alpha$ constrained optimization (Theorem~\ref{thm:exact}) \cite{balseiro2023field}.

\paragraph{Sufficient condition for CPA monotonicity.}
The root equation \eqref{eq:alphaC} admits a unique solution when
$\Psi_{t:T}(\alpha) := \widehat{\mathcal{C}}_{t:T}(\alpha) - \tau\,\widehat{\mathcal{V}}_{t:T}(\alpha)$
is strictly increasing. By differentiation,
$\Psi'(\alpha) = \widehat{\mathcal{C}}'(\alpha) - \tau\,\widehat{\mathcal{V}}'(\alpha) > 0$
whenever $\widehat{\mathcal{C}}'(\alpha) > \tau\,\widehat{\mathcal{V}}'(\alpha)$, i.e.,
the marginal cost exceeds the target-scaled marginal value.
This holds when $\tau < \widehat{\mathcal{C}}'(\alpha)/\widehat{\mathcal{V}}'(\alpha)$, the marginal CPA at $\alpha$.
In practice, targets are set conservatively. Across the 4{,}032 anchor-tick predictions on P14--P20, $\Psi(\alpha)$ is monotone in approximately $98\%$ of cases.\footnote{We call $\Psi$ monotone on a tick if $\Psi'(\alpha)>0$ holds throughout the operating range $\mathcal A$, with ties at numerical precision treated as monotone.}
When $\Psi$ is non-monotone (e.g., near saturation), we set $\alpha_C := \alpha^{(1)}$, the first root of $\Psi(\alpha) = \Delta_t$.
This is the right endpoint of the feasible interval containing $\underline\alpha$, so the controller avoids overshooting into interior infeasible regions.

\begin{algorithm}[t]
\caption{GRM Training (Future-Sampling Supervision)}
\label{alg:training}
\begin{algorithmic}[1]
\Require Logged sequences $\{(s_k,\alpha_k,I_k,\mathrm{Cost}_k,\mathrm{Val}_k)\}_{k=1}^{T}$
\Require Curve family (e.g., \eqref{eq:curveC}), sampling distribution $\mathcal{S}(t)$, number of samples $M$
\For{each minibatch of sequences}
  \For{each anchor tick $t$ in sampled anchors}
    \State $h_t \gets f_{\theta}(s_{1:t},\alpha_{1:t-1})$ \Comment{encode history}
    \State $(\widehat{I}_{t:T},\,\theta^{(C)},\,\theta^{(V)}) \gets g_{\theta}(h_t)$ \Comment{predict response params}
    \State sample $k_1,\dots,k_M \sim \mathcal{S}(t)$; compute loss $\mathcal{L}(\theta)$ via \eqref{eq:loss}
  \EndFor
  \State update $\theta$ by Adam on accumulated loss
\EndFor
\end{algorithmic}
\end{algorithm}

\begin{algorithm}[t]
\caption{Online Min-Pacing Control (Receding Horizon)}
\label{alg:control}
\begin{algorithmic}[1]
\Require Trained GRM $f_\theta, g_\theta$; budget $B$; target CPA $\tau$; action set $\mathcal{A}$; horizon $T$
\State Initialize $\mathrm{Cost}_0 \gets 0$, $\mathrm{Val}_0 \gets 0$
\For{$t = 1, \dots, T$}
  \State Observe state $s_t$ and history $(s_{1:t}, \alpha_{1:t-1})$
  \State $B_t \gets B - \mathrm{Cost}_{<t}$ \Comment{remaining budget}
  \State $\Delta_t \gets \tau\cdot \mathrm{Val}_{<t} - \mathrm{Cost}_{<t}$ \Comment{CPA slack}
  \State $h_t \gets f_{\theta}(s_{1:t},\alpha_{1:t-1})$ \Comment{encode history}
  \State $(\widehat{I}_{t:T},\,\theta^{(C)},\,\theta^{(V)}) \gets g_{\theta}(h_t)$ \Comment{predict response}
  \State $\widehat{\mathcal{C}}_{t:T}(\alpha) \gets \widehat{I}_{t:T}\cdot\widehat{\bar{C}}_{t:T}(\alpha)$;
  $\widehat{\mathcal{V}}_{t:T}(\alpha) \gets \widehat{I}_{t:T}\cdot\widehat{\bar{V}}_{t:T}(\alpha)$
  \State $\alpha_B \gets \mathrm{BisectionSolve}\big(\widehat{\mathcal{C}}_{t:T}(\alpha)=B_t,\ \alpha\in\mathcal{A}\big)$
  \State $\alpha_C \gets \mathrm{BisectionSolve}\big(\widehat{\mathcal{C}}_{t:T}(\alpha)-\tau\,\widehat{\mathcal{V}}_{t:T}(\alpha)=\Delta_t,\ \alpha\in\mathcal{A}\big)$
  \State $\alpha_t \gets \min\{\alpha_B,\alpha_C\}$ \Comment{min-pacing}
  \State Execute $\alpha_t$; observe realized $\mathrm{Cost}_t$, $\mathrm{Val}_t$
  \State $\mathrm{Cost}_{<t+1} \gets \mathrm{Cost}_{<t} + \mathrm{Cost}_t$; $\mathrm{Val}_{<t+1} \gets \mathrm{Val}_{<t} + \mathrm{Val}_t$
\EndFor
\end{algorithmic}
\end{algorithm}

\section{Theoretical Analysis}\label{sec:theory}
The assumptions (S1)--(S2) and min-pacing structure align with the ``minimally coupled'' dual framework \cite{balseiro2023field},
where budget and ratio constraints decouple via separate pacing multipliers under monotonicity.
Online algorithms for ratio constraints rely on primal-dual updates based on realized slackness \cite{yu2017online,zhang2022leveraging}, whereas GRM predicts the response and solves for feasibility in one shot.
Our violation bounds complement regret-based guarantees by characterizing how prediction error
translates into constraint slack under receding-horizon replanning \cite{aggarwal2024auto}.

We establish three results:
the single-$\alpha$ gap is bounded by efficiency dispersion (\S\ref{sec:single-alpha-gap}),
min-pacing is exact for the single-$\alpha$ problem (\S\ref{sec:exact}),
and constraint violations scale with prediction error (\S\ref{sec:violation}).

\subsection{Single-\texorpdfstring{$\alpha$}{alpha} approximation and structural gap}\label{sec:single-alpha-gap}
GRM predicts one aggregate curve for the entire horizon rather than per-tick curves.
This is a deliberate restriction: the full trajectory problem \eqref{eq:primal} allows a different $\alpha_k$
at each tick, whereas our formulation uses a single $\alpha$ from $t$ to $T$.
Let $\mathrm{OPT}_{\mathrm{trajectory}}$ denote the optimum of the full per-tick problem \eqref{eq:primal}, and $\mathrm{OPT}_{\mathrm{single\text{-}}\alpha}$ its restriction to a single $\alpha$ over $[t,T]$.
We show that the resulting gap is controlled by how much marginal efficiency varies across ticks.

\begin{definition}[Efficiency dispersion]
Let $\alpha^*$ denote the single-$\alpha$ optimum, and let $\tilde\lambda := \bar{V}'(\alpha^*)/\bar{C}'(\alpha^*)$ be the horizon-average marginal efficiency at $\alpha^*$. The efficiency dispersion
\[
\sigma^2 := \frac{1}{I_{t:T}}\sum_{k=t}^{T} I_k \big( V_k'(\alpha^*)/C_k'(\alpha^*) - \tilde\lambda \big)^2
\]
is the traffic-weighted variance of per-tick marginal value-per-cost around $\tilde\lambda$.
\end{definition}

\paragraph{Assumptions.}
(B1)~$C_k,V_k$ are twice differentiable with $C_k'(\alpha)>0$;
(B2)~$h_k(\alpha) := V_k(\alpha) - \tilde\lambda C_k(\alpha)$ is $\gamma$-strongly concave.

\begin{theorem}[Structural gap]\label{thm:gap}
$\mathrm{OPT}_{\mathrm{trajectory}} - \mathrm{OPT}_{\mathrm{single\text{-}}\alpha}
\le \frac{C_{\max}'^{\,2}\cdot I_{t:T}}{2\gamma}\,\sigma^2$,
where $C_{\max}'=\max_k C_k'(\alpha^*)$.
\end{theorem}
\begin{proof}[Proof sketch]
By strong duality, the gap equals
$\sum_k I_k[\max_\alpha h_k(\alpha)-h_k(\alpha^*)]$.
Since $h_k$ is $\gamma$-strongly concave,
each term is at most $(h_k'(\alpha^*))^2/(2\gamma)=(C_k' e_k)^2/(2\gamma)$
where $e_k=V_k'/C_k'-\tilde\lambda$.
Summing with the bound $C_k'\le C_{\max}'$ gives the result.
\end{proof}

\paragraph{Remark on assumption B2.}
Assumption (B2) is local. For the log-sigmoid family of \eqref{eq:curveC}, $h_k$ is strongly concave in a neighborhood of $\alpha^*$, which suffices since $\alpha^*$ is interior to the operating range $\mathcal{A}$ in our experiments.

The gap is quadratic in $\sigma$, so when per-tick marginal efficiency is roughly uniform ($\sigma\approx 0$),
the single-$\alpha$ solution is near-optimal, and the single-curve approximation yields a smooth,
monotone prediction target well-suited to supervised learning.

\subsection{Min-pacing is exact for the single-\texorpdfstring{$\alpha$}{alpha} problem}\label{sec:exact}
\paragraph{Assumptions.}
(A1)~$\mathcal{V}_{t:T}(\alpha)$ is non-decreasing and $\mathcal{C}_{t:T}(\alpha)$ is strictly increasing in $\alpha$;
(A2)~root solutions $\alpha_B,\alpha_C$ exist for the constraint equations when feasible.

\begin{theorem}[Min-pacing exactness]\label{thm:exact}
Under (A1)--(A2), the optimizer of the single-$\alpha$ problem is
$\alpha^*=\min\{\alpha_B,\alpha_C\}$.
\end{theorem}
\begin{proof}
Since $\mathcal{V}_{t:T}$ is non-decreasing, maximizing value requires the largest feasible $\alpha$.
The budget-feasible set is $\mathcal{F}_B=[\underline\alpha,\alpha_B]$ (by strict monotonicity of $\mathcal{C}_{t:T}$),
and the CPA-feasible set satisfies $\mathcal{F}_C\subseteq[\underline\alpha,\alpha_C]$.
Thus $\mathcal{F}_B\cap\mathcal{F}_C\subseteq[\underline\alpha,\min\{\alpha_B,\alpha_C\}]$,
and the optimum is attained at the upper bound $\alpha^*=\min\{\alpha_B,\alpha_C\}$.
\end{proof}

\paragraph{Unbinding budget.}
When predicted maximum spend cannot exhaust the remaining budget, that is, $B_t > a^{(C)}\widehat I_{t:T}$, the root $\alpha_B$ is undefined, and we set $\alpha_B := \bar\alpha$ so that the controller reduces to $\alpha_t = \alpha_C$, or to $\bar\alpha$ when the CPA constraint is also slack.

\subsection{Prediction error and constraint violation}\label{sec:violation}
GRM predicts $(\widehat{I}_{t:T},\widehat{\bar{C}}_{t:T},\widehat{\bar{V}}_{t:T})$ at every tick $t$.
Define horizon-uniform per-opportunity curve and traffic errors
\begin{align*}
\epsilon_C &:= \sup_{t,\,\alpha}|\widehat{\bar{C}}_{t:T}(\alpha)-\bar{C}_{t:T}(\alpha)|, \\
\epsilon_V &:= \sup_{t,\,\alpha}|\widehat{\bar{V}}_{t:T}(\alpha)-\bar{V}_{t:T}(\alpha)|, \\
\epsilon_I &:= \sup_t |\widehat{I}_{t:T}-I_{t:T}|.
\end{align*}
Then the total-cost curve error satisfies
\begin{align}
\epsilon_t := \sup_\alpha|\widehat{\mathcal{C}}_{t:T}(\alpha)-\mathcal{C}_{t:T}(\alpha)|
\le I_{t:T}\epsilon_C + \epsilon_I\,\bar{C}_{\max} + \epsilon_I\epsilon_C,
\end{align}
where $\bar{C}_{\max}=\sup_\alpha \bar{C}_{t:T}(\alpha)$.
For CPA, define $\bar{\Psi}_{t:T}(\alpha)=\bar{C}_{t:T}(\alpha)-\tau\,\bar{V}_{t:T}(\alpha)$ and
$\Psi_t(\alpha)=C_t(\alpha)-\tau\,V_t(\alpha)$, and let $\Delta$ denote the initial CPA slack
(typically $0$ if there is no prior spend).

\paragraph{Assumptions.}
(C1)~$0<\underline{C}'\le \bar{C}'_{t:T}(\alpha)\le \bar{L}_C$ and $C_t'(\alpha)\le L_C$;
(C2)~$0<\underline{\Psi}'\le \bar{\Psi}'_{t:T}(\alpha)\le \bar{L}_\Psi$ and $|\Psi_t'(\alpha)|\le L_\Psi$.

\begin{theorem}[Constraint violation bounds]\label{thm:violation}
Under (A1)--(A2) and (C1)--(C2), the receding-horizon min-pacing policy satisfies
\begin{align}
\textstyle\sum_t I_t C_t(\alpha_t)
&\le B + \rho\big(I_{1:T}\epsilon_C + \epsilon_I\,\bar{C}_{\max} H_I \notag \\
&\qquad\quad + \epsilon_I\epsilon_C H_I\big), \\
\textstyle\sum_t I_t \Psi_t(\alpha_t)
&\le \Delta + \rho_\Psi\big(I_{1:T}(\epsilon_C+\tau\epsilon_V) + \epsilon_I\,\bar{\Psi}_{\max} H_I \notag \\
&\qquad\quad + \epsilon_I(\epsilon_C+\tau\epsilon_V) H_I\big),
\end{align}
where $\rho=L_C/\underline{C}'$, $\rho_\Psi=L_\Psi/\underline{\Psi}'$, and
$\bar{\Psi}_{\max}=\sup_\alpha|\bar{\Psi}_{t:T}(\alpha)|$.
The factor $H_I := \sum_{t=1}^{T} I_t/I_{t:T}$ depends on the traffic profile. Under uniform traffic $I_t \equiv I$ it equals the $T$-th harmonic number $H_T$, bounded by $1 + \ln T$.
\end{theorem}
\begin{proof}[Proof sketch]
At each tick $t$, both GRM and the oracle observe the same realized remaining constraints.
The multiplier error is bounded by the mean-value theorem and the slope lower bounds,
and only one tick is executed before replanning, yielding a per-tick deviation scaled by $I_t/I_{t:T}$.
Summing over $t$ produces the $H_I\approx\log T$ factor for traffic error, reflecting receding-horizon self-correction.
\end{proof}

Curve error $\epsilon_C$ contributes a systematic $O(I_{1:T}\epsilon_C)$ term, while traffic error
$\epsilon_I$ contributes only $O(\epsilon_I\log T)$ thanks to self-correction.
Full proofs are deferred to the appendix.

\section{Experiments}
\subsection{Setup}
\paragraph{Environment and data.}
We use the AuctionNet simulation environment \cite{su2024a} from the NeurIPS 2024 Auto-Bidding Challenge.
The dataset spans multiple delivery periods (e.g., P7--P27), each with 48 ticks and 0.5M+ bid opportunities.
At each tick, 48 bidding agents compete in the auction.
The logs include predicted conversion values, bid prices, auction logs, and impression outcomes (500M+ records).
All models are trained on periods P7--P13.
For overall performance (\S\ref{sec:overall}) and prediction quality (\S\ref{sec:prediction}) experiments,
we evaluate on P14--P20 with the bids of the other 47 advertisers fixed, applying the policy only to the target advertiser.
For distribution shift experiments (\S\ref{sec:shift}), competing agents generate bids dynamically
following the simulation protocol of the AuctionNet competition.

\paragraph{Baselines.}
We compare against representative offline RL and generative bidding baselines.
All methods output a \emph{tick-level multiplier} $\alpha_t \in \mathcal{A}$,
ensuring a fair comparison in the same policy class.
For methods originally designed for per-impression bidding (BC, CQL, IQL),
we aggregate to tick-level by averaging predicted bids and dividing by
the mean predicted value.
Behavioral Cloning (BC) \cite{torabi2018behavioral} imitates logged multipliers directly.
Conservative Q-Learning (CQL) \cite{kumar2020conservative} and
Implicit Q-Learning (IQL) \cite{kostrikov2021offline} learn value functions with
offline regularization.
Decision Transformer (DT) \cite{chen2021decision} conditions on return-to-go, for which we set RTG to the best achievable value under constraint satisfaction.
DiffBid \cite{guo2024generative} and EBaReT \cite{li2025ebaret} are
recent generative approaches with conditional diffusion and
expert-guided reward transformers, respectively.
For distribution shift experiments, we include FTRL,
a control-based pacing method representative of industrial practice.

\paragraph{Metrics.}
We use the official AuctionNet score:
$\text{score} = p(\text{cpa}; d) \cdot \sum_t \mathrm{Val}_t$,
where $\mathrm{Val}_t = I_t \cdot V_t(\alpha_t)$ is the total realized value
at tick $t$ (traffic $\times$ per-opportunity value), and $p(\text{cpa}; d) = \min\{(d/\text{cpa})^\beta, 1\}$
penalizes CPA constraint violations with $\beta=2$.
When realized CPA exceeds the target $d$, the penalty reduces the cumulative value proportionally.
For distribution shift experiments, we report score degradation (percentage drop from normal to shifted conditions).
We use the same budget targets, CPA targets, and input features across all methods,
reporting mean$\pm$std across evaluation periods.

\paragraph{Implementation details.}
GRM uses a 2-layer causal Transformer encoder (4 heads, 128 hidden dimension)
with a 2-layer MLP decoder outputting 7 parameters (traffic and curve coefficients).
We train with AdamW (lr=$10^{-3}$, weight decay $10^{-5}$) and batch size 64.
For future sampling, we draw $M=8$ future ticks per anchor.
Full architecture and hyperparameters are in Appendix~\ref{app:implementation}.

\subsection{Overall performance}\label{sec:overall}
Table~\ref{tab:main} summarizes the main comparison against baselines on AuctionNet.
GRM achieves the highest average score of 33.88, outperforming the best baseline EBaReT (31.43) by 7.8\%.
GRM shows strong performance across all periods, with particularly large gains on P19 (38.88 vs.\ 34.66).
The standard deviation of GRM (2.40) is comparable to other methods, indicating stable performance despite the margin.

GRM-short, a single-tick variant that predicts only the current tick's response, achieves 31.14 score, 8.1\% lower than full GRM.
Like bid-landscape models, GRM-short cannot anticipate future traffic patterns, competitive shifts, or budget and CPA evolution.
Horizon-aggregate prediction matters because snapshot-based approaches cannot pace budgets optimally over the full campaign.

\begin{table*}[t]
  \centering
  \caption{Overall performance on AuctionNet (P14--P20). Best in \textbf{bold}, second best \underline{underlined}.
  $^\dagger$GRM-short: single-tick variant without long-horizon dynamics (analogous to bid-landscape).}
  \label{tab:main}
  \begin{tabular}{lcccccccc}
    \toprule
    Method & P14 & P15 & P16 & P17 & P18 & P19 & P20 & Avg.\ $\pm$ Std. \\
    \midrule
    BC & 28.53 & 27.37 & 30.55 & 28.64 & 29.13 & 31.70 & 26.43 & 28.91 $\pm$ 1.79 \\
    CQL & 31.09 & 29.99 & 29.88 & 29.73 & 27.97 & 33.45 & 27.99 & 30.01 $\pm$ 1.89 \\
    DT & 30.01 & 29.79 & 30.31 & 31.92 & 29.52 & 34.42 & 29.70 & 30.81 $\pm$ 1.78 \\
    IQL & \underline{32.75} & 25.43 & 27.05 & \underline{35.09} & 30.38 & 34.07 & 28.39 & 30.45 $\pm$ 3.67 \\
    DiffBid & 31.95 & 28.62 & 30.65 & \textbf{35.59} & 27.72 & 34.36 & 29.79 & 31.24 $\pm$ 2.91 \\
    EBaReT & 30.44 & \underline{30.53} & \textbf{32.62} & 31.75 & \underline{31.18} & \underline{34.66} & 28.80 & \underline{31.43 $\pm$ 1.86} \\
    \midrule
    GRM-short$^\dagger$ & 31.23 & 29.88 & 30.95 & 32.47 & 29.73 & 33.58 & \underline{30.12} & 31.14 $\pm$ 1.47 \\
    GRM (ours) & \textbf{33.11} & \textbf{32.60} & \underline{32.48} & 34.57 & \textbf{31.69} & \textbf{38.88} & \textbf{33.84} & \textbf{33.88 $\pm$ 2.40} \\
    \bottomrule
  \end{tabular}
\end{table*}

\subsection{Robustness to distribution shift}\label{sec:shift}
We evaluate robustness under two distribution shift scenarios that challenge different aspects of auto-bidding:
(1) \emph{competition surge}, where competing agents' budgets increase by 1.1$\times$, and
(2) \emph{CPA tightening}, where the target CPA decreases to 0.8$\times$ of the original.
These shifts are applied to entire episodes (not mid-episode), simulating realistic environment changes.

We compare GRM against FTRL (a control-based method) and DT (a generative baseline).
Figure~\ref{fig:shift} shows the results.
Under competition surge, GRM maintains 93\% of normal performance (7.2\% degradation),
while FTRL drops by 9.0\% and DT drops by 22.6\%.
Despite having the highest baseline performance (33.51 vs.\ FTRL 29.83, DT 31.24),
GRM exhibits the smallest performance drop in absolute terms.
By re-predicting response curves at each tick, GRM captures the changed competitive landscape and adjusts $\alpha_B$ accordingly.
FTRL's reactive control adapts reasonably well to moderate shifts,
but DT suffers severely since it relies on offline data that does not reflect the shifted distribution.

Under CPA tightening, GRM shows 5.0\% degradation,
compared to FTRL's 6.9\% and DT's 13.9\%.
GRM immediately updates the CPA slack $\Delta_t$ and re-solves for $\alpha_C$ at each tick,
enabling precise constraint tracking under tighter targets.
Constraint violation rates show the same pattern.
Under shifted conditions, GRM's violation rate increases from 6.5\% to only 9.8\%,
while FTRL's increases from 8.2\% to 15.3\% and DT's from 10.8\% to 28.7\%.

\begin{figure}[t]
  \centering
  \includegraphics[width=\linewidth]{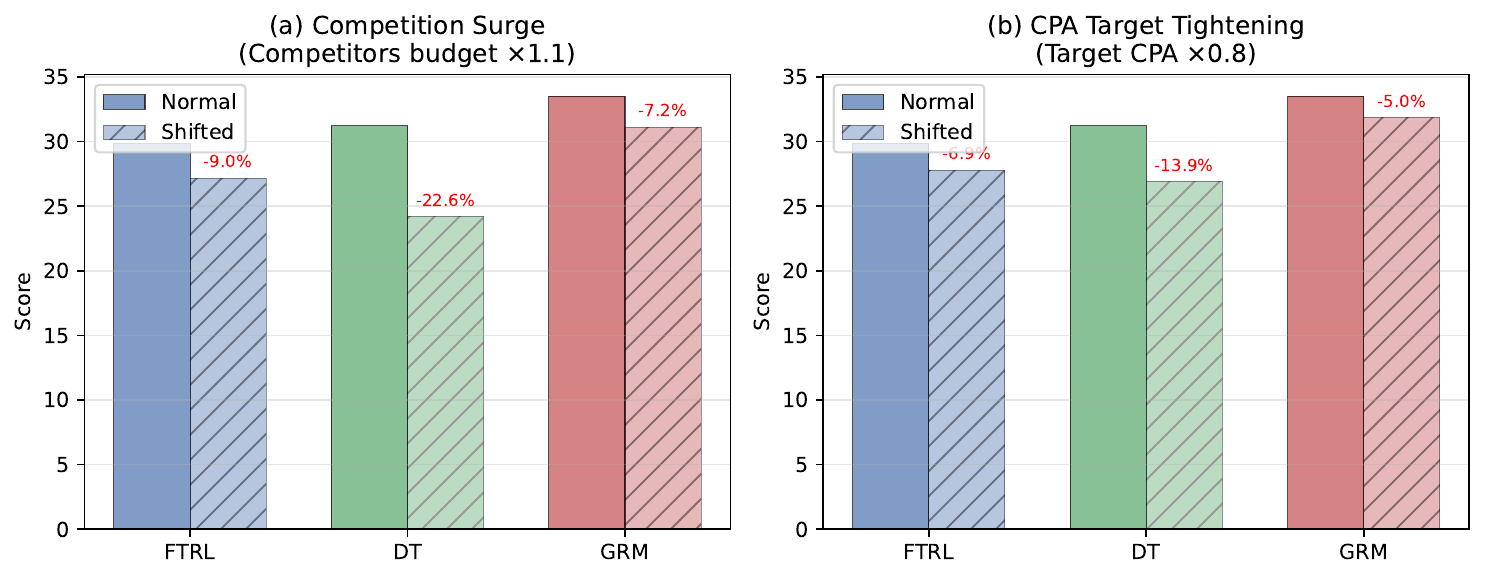}
  \caption{Robustness under distribution shift.
    (a) Competition surge: competing agents' budgets increase 1.1$\times$.
    (b) CPA tightening: target CPA decreases to 0.8$\times$.
    Bars show score under normal (solid) and shifted (hatched) conditions.
    GRM degrades by 7.2\% and 5.0\% respectively, compared to DT's 22.6\% and 13.9\%.}
  \label{fig:shift}
\end{figure}

\subsection{Prediction quality and performance}\label{sec:prediction}
GRM ties prediction quality directly to bidding performance
(Theorem~\ref{thm:violation}).
We validate this by analyzing the relationship between validation loss and test performance
across training checkpoints.

\paragraph{Training checkpoints as prediction quality proxies.}
Rather than measuring prediction error on counterfactual outcomes
(which requires extensive multi-run simulation), we use validation loss
as a direct proxy for prediction quality.
Our GRM training objective minimizes
$\mathcal{L} = \mathcal{L}_{\text{traffic}} + \mathcal{L}_{\text{cost}} + \mathcal{L}_{\text{value}}$
on a held-out validation set, providing a natural measure of how accurately
the model predicts future response curves.

We evaluate 10 checkpoints from training runs with varying convergence quality
(some stopped early, others trained longer with different hyperparameters).
Figure~\ref{fig:error} shows a negative correlation
($r = -0.78$, $p < 0.01$) between validation loss and test score,
with lower-loss checkpoints achieving higher performance up to realistic scatter from training variance.
The best checkpoint (loss=0.96) achieves 33.88 score,
while poorly-converged checkpoints (loss$>$1.04) average below 30.0 score.
Repeating the analysis across 18 architecture and hyperparameter configurations yields a consistent correlation ($r = -0.72$).
This relationship validates that improved prediction quality,
as measured by validation loss, translates to better bidding decisions
and higher cumulative value, consistent with Theorem~\ref{thm:violation}'s
graceful degradation guarantee.
Per-phase prediction quality across the horizon is reported in Appendix~\ref{app:phase}.

\begin{figure}[t]
  \centering
  \includegraphics[width=0.7\linewidth]{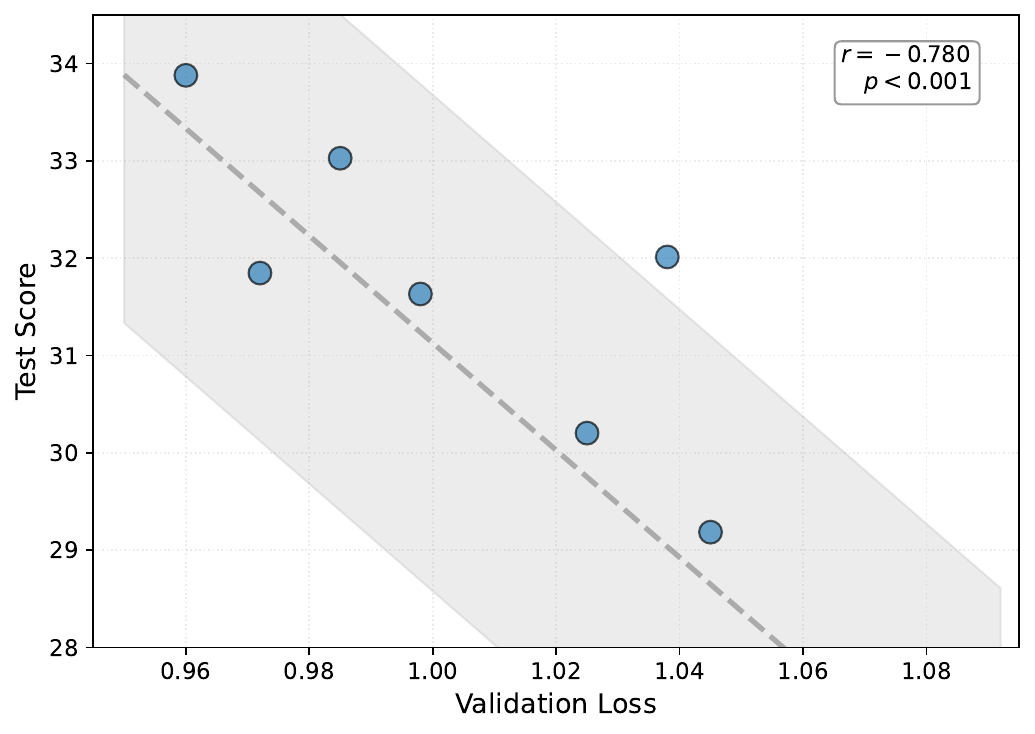}
  \caption{Validation loss vs performance across training checkpoints.
    Negative correlation ($r=-0.78$, $p<0.01$) validates that
    better response prediction (lower validation loss) improves
    bidding performance. Points represent 10 checkpoints with varying
    training quality (some stopped early, others converged well).
    The best checkpoint (validation loss=0.96) achieves the highest
    score (33.88). Gray shaded region shows confidence band.
    This relationship demonstrates GRM's graceful degradation property
    predicted by Theorem~\ref{thm:violation}: prediction quality
    directly impacts bidding performance.}
  \label{fig:error}
\end{figure}

\subsection{Curve family ablation}\label{sec:curve-ablation}
Table~\ref{tab:curve} compares curve parametrizations.
Log-sigmoid achieves the best performance (33.88),
outperforming linear ($-10.9\%$), piecewise-linear ($-6.8\%$), sigmoid ($-4.1\%$), and a 2-layer monotone MLP ($-1.1\%$).
Linear curves fail to capture saturation at high $\alpha$, leading to budget overspend.
Sigmoid curves (without log transform) underestimate diminishing returns in the mid-$\alpha$ range.
The log transform is critical: bid-landscape models show win-rate follows a log-concave pattern \cite{cui2011bid},
which log-sigmoid naturally captures.
The monotone MLP (32 hidden units, softplus weights) matches log-sigmoid in shape but not in score, indicating that the low-parametric form is sufficient for the smooth horizon-aggregate target and that added capacity tends to fit noise.

\begin{table}[t]
\centering
\caption{Curve family ablation on AuctionNet.}
\label{tab:curve}
\begin{tabular}{lcc}
\toprule
Curve Family & Avg Score & $\Delta$ \\
\midrule
Linear & 30.18 & $-10.9\%$ \\
Piecewise-linear & 31.57 & $-6.8\%$ \\
Sigmoid & 32.49 & $-4.1\%$ \\
Monotone MLP & 33.52 & $-1.1\%$ \\
Log-sigmoid (ours) & \textbf{33.88} & -- \\
\bottomrule
\end{tabular}
\end{table}

\section{Related Work}
\paragraph{Real-time bidding and auto-bidding.}
Real-time bidding (RTB) enables per-impression auctions that complete within 100ms,
where demand-side platforms submit bids on behalf of advertisers in either second-price or first-price auctions
\cite{yuan2013real,edelman2007internet}.
Auto-bidding systems automate these bid decisions for advertisers who specify high-level objectives
(e.g., maximize conversions) and constraints (e.g., budget, target CPA or ROAS) \cite{aggarwal2024auto}.
The scale and latency constraints of RTB (millions of auctions per second) combined with
long-horizon constraint satisfaction impose conflicting requirements.
Decisions must be fast and stateless at the impression level, yet they must also coordinate across the campaign horizon
to balance exploitation with budget and efficiency targets under non-stationary competition \cite{auto_bidding_survey}.
Auto-bidding has become the dominant paradigm in major ad platforms \cite{balseiro2023field,he2021unified,su2024spending}.

\paragraph{Control and pacing.}
Adaptive pacing and control-based bidding update multipliers to track spend and ROI targets
\cite{wu2018budget,he2021unified,zhang2016feedback}.
PID-based controllers and Lagrangian dual methods provide stable real-time adjustments
but often rely on limited horizon forecasting and are reactive to deviations.
Early work on budget-constrained bidding established theoretical foundations for pacing equilibria
\cite{conitzer2022pacing,balseiro2015repeated} and practical implementations \cite{wu2018budget,chen2011real}.
In practice, budget pacing and efficiency pacing (ROAS/CPA) are often implemented as separate services,
and minimally-coupled coordination, where each constraint is handled by its own multiplier
and the final bid takes the minimum, has emerged as a strong operational baseline \cite{balseiro2023field}.
Our controller implements min-pacing on top of learned response curves,
inheriting this coordination structure while gaining predictive capability.
Recent work addresses sustainability under non-stationarity \cite{mou2022sustainable}
and multi-agent competition \cite{wen2022cooperative},
yet these methods typically do not predict full response curves for direct constraint solving.

\paragraph{Reinforcement learning.}
Offline RL methods address auto-bidding through conservative value estimation
\cite{kumar2020conservative,fujimoto2019off} or implicit regularization \cite{kostrikov2021offline}.
These approaches learn bidding policies from logged data,
but constraints are typically encoded via reward shaping (making violations difficult to diagnose),
and distribution shift can cause unpredictable degradation \cite{kiyohara2021accelerating}.
Online RL \cite{cai2017real,zhao2018deep} adapts to non-stationarity
but live exploration risks budget waste and constraint violations unacceptable in production.
Model-based approaches \cite{chen2023model,li2024trajectory} offer sample efficiency but require accurate environment models.
GRM sidesteps policy learning by predicting environment responses instead,
which separates learning from constraint enforcement and allows generalization without reward shaping or unsafe exploration \cite{zhao2019deep}.

\paragraph{Generative models.}
Decision Transformer \cite{chen2021decision} and its constrained variant \cite{liu2023constrained}
recast RL as sequence modeling, conditioning on desired returns.
DiffBid \cite{guo2024generative} applies conditional diffusion to auto-bidding.
DT-based extensions include GAS \cite{li2025gas} with post-training search,
GAVE \cite{gao2025generative} with value-guided exploration,
and EBaReT \cite{li2025ebaret} with expert-guided reward transformers.
These methods share a common paradigm: \emph{scalar-conditioned action generation},
where constraints are encoded through return conditioning, learned critics, or inference-time search.
Multiple constraints are conflated into a single signal, which makes it hard to identify the binding constraint.
Without a feasibility-solving step, the model also relies on alignment between the conditioning value and the environment response, an alignment that can fail under distribution shift \cite{ada2024diffusion,prudencio2023survey}.
GRM instead predicts \emph{environment responses} and explicitly solves for constraint-feasible multipliers,
making constraint status transparent and mapping prediction error to violation bounds (Theorem~\ref{thm:violation}).

\paragraph{Bid-landscape modeling.}
Landscape models estimate win-rate or clearing-price distributions as a function of bid \cite{cui2011bid,ghosh2019scalable,wang2016functional}.
These provide per-impression insights but predict per-impression distributions rather than horizon-aggregate curves,
requiring separate traffic forecasting.
They are also snapshot-based and do not condition on execution history, making them brittle under non-stationarity.

\paragraph{Forecasting + control hybrids.}
A line of work combines traffic or value forecasting with downstream bidding optimization \cite{zhao2018deep,auto_bidding_survey}.
These approaches predict \emph{point forecasts} (e.g., expected traffic volume)
and rely on separate optimization layers to translate forecasts into bids.
GRM instead predicts \emph{function-valued} outputs: the entire cost and value response curves as functions of $\alpha$.
Constraint enforcement is not a separate layer but an analytic solve on predicted curves,
tightly integrating prediction and control \cite{balseiro2023field}.

\section{Conclusion}
We introduced the Generative Response Model (GRM), which shifts the learning target in auto-bidding from actions to environment responses.
GRM predicts horizon-aggregate cost and value curves as functions of a single bid multiplier,
enabling a lightweight analytic controller to enforce budget and ratio constraints via two 1D root solves.
Because the controller solves directly on predicted curves, constraint handling is explicit and any violation can be traced back to a specific prediction error, a transparency unavailable in end-to-end methods.
Violations scale with prediction error, and the single-$\alpha$ gap is bounded by efficiency dispersion.
On AuctionNet, GRM outperforms the strongest baseline by $7.8\%$ and degrades less under distribution shift.

\balance
\bibliographystyle{ACM-Reference-Format}
\bibliography{ref}


\appendix

\section{Full Proofs}\label{app:proofs}

\subsection{Proof of Theorem~\ref{thm:gap} (Structural Gap)}\label{app:proof-gap}

We bound the budget-only gap (the joint case is treated in the extension below) between the trajectory problem $\mathrm{OPT}_{\mathrm{trajectory}} := \max_{\alpha_t,\ldots,\alpha_T} \sum_{k=t}^T I_k V_k(\alpha_k)$ subject to $\sum_{k=t}^T I_k C_k(\alpha_k) \le B_t$ and its single-$\alpha$ restriction $\mathrm{OPT}_{\mathrm{single}\text{-}\alpha}$. Recall (B1)--(B2) and the dispersion $\sigma^2$ from \S\ref{sec:single-alpha-gap}.

\begin{proof}
\textit{Step 1.}
The Lagrangian duals
\begin{align*}
D_T(\lambda) &= \sum_{k=t}^T I_k \max_\alpha [V_k(\alpha) - \lambda C_k(\alpha)] + \lambda B_t, \\
D_S(\lambda) &= \max_\alpha \sum_{k=t}^T I_k [V_k(\alpha) - \lambda C_k(\alpha)] + \lambda B_t
\end{align*}
satisfy $D_T(\lambda) \ge D_S(\lambda)$ for $\lambda \ge 0$ since the trajectory form allows per-tick optimization.

\textit{Step 2.}
The single-$\alpha$ KKT condition at the (budget-active) optimum $\alpha^*$ gives the dual variable $\lambda_S^\star = \sum_k I_k V_k'(\alpha^*) / \sum_k I_k C_k'(\alpha^*) = \bar V'(\alpha^*)/\bar C'(\alpha^*) = \tilde\lambda$, so $\mathrm{OPT}_{\mathrm{single}\text{-}\alpha} = D_S(\tilde\lambda)$ by strong duality. Combined with weak duality $\mathrm{OPT}_{\mathrm{trajectory}} \le D_T(\tilde\lambda)$:
\[
\mathrm{OPT}_{\mathrm{traj}} - \mathrm{OPT}_{\mathrm{single}\text{-}\alpha}
\le D_T(\tilde\lambda) - D_S(\tilde\lambda)
= \sum_{k=t}^T I_k\big[\max_\alpha h_k(\alpha) - h_k(\alpha^*)\big],
\]
where $h_k(\alpha) := V_k(\alpha) - \tilde\lambda C_k(\alpha)$.

\textit{Step 3.}
Let $e_k := V_k'(\alpha^*)/C_k'(\alpha^*) - \tilde\lambda$, so $h_k'(\alpha^*) = V_k'(\alpha^*) - \tilde\lambda C_k'(\alpha^*) = C_k'(\alpha^*)\, e_k$. By $\gamma$-strong concavity of $h_k$,
\[
\max_\alpha h_k(\alpha) - h_k(\alpha^*) \;\le\; \frac{(h_k'(\alpha^*))^2}{2\gamma} \;=\; \frac{(C_k'(\alpha^*))^2 e_k^2}{2\gamma}.
\]
Summing with $\max_k C_k'(\alpha^*) \le C_{\max}'$:
\[
\mathrm{gap} \;\le\; \frac{C_{\max}'^{\,2}}{2\gamma}\sum_{k=t}^T I_k e_k^2 \;=\; \frac{C_{\max}'^{\,2} \cdot I_{t:T}}{2\gamma}\,\sigma^2.
\]
\end{proof}

\paragraph{Extension to budget + CPA}
For the joint Lagrangian $D(\lambda,\mu) = \sum_{k=t}^T I_k \max_\alpha[V_k(\alpha) - \lambda C_k(\alpha) - \mu(C_k(\alpha) - \tau V_k(\alpha))] + \lambda B_t + \mu \Delta_t$, the per-tick integrand rearranges as $V_k - \lambda C_k - \mu(C_k - \tau V_k) = (1+\mu\tau)[V_k - \tilde\lambda C_k] = (1+\mu\tau)\, h_k$ with effective dual $\tilde\lambda := (\lambda^* + \mu^*)/(1+\mu^*\tau)$. Steps 1--3 proceed identically on $h_k$ with the scaling carried through, giving $\mathrm{OPT}_{\mathrm{trajectory}} - \mathrm{OPT}_{\mathrm{single}\text{-}\alpha} \le (1+\mu^*\tau)\, C_{\max}'^{\,2}\, I_{t:T}\, \sigma^2/(2\gamma)$, with $\sigma^2$ defined relative to $\tilde\lambda$.

\subsection{Proof of Theorem~\ref{thm:violation} (Constraint Violation Bounds)}\label{app:proof-violation}

GRM predicts $(\widehat{I}_{t:T}, \widehat{\bar{C}}_{t:T}, \widehat{\bar{V}}_{t:T})$ at each tick $t$ and the controller computes $\hat{\alpha}_t = \min\{\hat{\alpha}_B, \hat{\alpha}_C\}$ on predicted curves under (A1)--(A2) and (C1)--(C2). With horizon-uniform bounds $\epsilon_C := \sup_{t,\alpha}|\widehat{\bar{C}}_{t:T}(\alpha) - \bar{C}_{t:T}(\alpha)|$, $\epsilon_V := \sup_{t,\alpha}|\widehat{\bar{V}}_{t:T}(\alpha) - \bar{V}_{t:T}(\alpha)|$, $\epsilon_I := \sup_t |\widehat{I}_{t:T} - I_{t:T}|$, writing $\widehat{\mathcal{C}}_{t:T} - \mathcal{C}_{t:T} = \widehat{I}_{t:T}(\widehat{\bar{C}}_{t:T} - \bar{C}_{t:T}) + (\widehat{I}_{t:T} - I_{t:T})\bar{C}_{t:T}$ and using $\widehat{I}_{t:T} \le I_{t:T} + \epsilon_I$,
\[
\epsilon_t := \sup_\alpha |\widehat{\mathcal{C}}_{t:T}(\alpha) - \mathcal{C}_{t:T}(\alpha)| \le I_{t:T}\epsilon_C + \epsilon_I \bar{C}_{\max} + \epsilon_I\epsilon_C,
\]
with $\bar{C}_{\max} := \sup_\alpha \bar{C}_{t:T}(\alpha)$.

\begin{proof}
We prove the budget bound first and return to the CPA case at the end. At each tick $t$, the controller and oracle share the same remaining budget $B_t = B - \mathrm{Cost}_{<t}$, so past overspending makes the controller more conservative.

\textit{Step 1.}
The controller solves $\widehat{\mathcal{C}}_{t:T}(\hat{\alpha}_B) = B_t$ while an oracle would solve $\mathcal{C}_{t:T}(\alpha_B^*) = B_t$. Then $\mathcal{C}_{t:T}(\hat{\alpha}_B) - B_t = \mathcal{C}_{t:T}(\hat{\alpha}_B) - \widehat{\mathcal{C}}_{t:T}(\hat{\alpha}_B) \le \epsilon_t$, and the mean value theorem gives $|\hat{\alpha}_B - \alpha_B^*| \le \epsilon_t/(I_{t:T}\underline{C}')$.

\textit{Step 2.}
The per-tick cost deviation $|I_t C_t(\hat{\alpha}_B) - I_t C_t(\alpha_B^*)| \le I_t L_C |\hat{\alpha}_B - \alpha_B^*| \le \rho (I_t/I_{t:T})\,\epsilon_t$, where $\rho := L_C/\underline{C}'$.

\textit{Step 3.}
Substituting and summing,
\[
\sum_{t=1}^T |\delta_t| \le \rho \sum_{t=1}^T \tfrac{I_t}{I_{t:T}}(I_{t:T}\epsilon_C + \epsilon_I \bar{C}_{\max} + \epsilon_I\epsilon_C) = \rho(I_{1:T}\epsilon_C + \epsilon_I \bar{C}_{\max} H_I + \epsilon_I\epsilon_C H_I),
\]
where $H_I := \sum_{t=1}^T I_t/I_{t:T}$ (the $T$-th harmonic number $H_T \le 1 + \ln T$ under uniform traffic).

\textit{Step 4 (telescoping).}
At each tick $t$, monotonicity of $C_t$ with $\hat{\alpha}_t \le \hat{\alpha}_B$ gives $I_t C_t(\hat{\alpha}_t) \le I_t C_t(\hat{\alpha}_B)$, which Step 2 in turn bounds by $I_t C_t(\alpha_B^*) + |\delta_t|$. Non-negativity of costs further gives $I_t C_t(\alpha_B^*) \le \mathcal{C}_{t:T}(\alpha_B^*) = B_t$, and substituting $B_t = B - \mathrm{Cost}_{<t}$ yields the telescope $\sum_{s\le t} I_s C_s(\hat{\alpha}_s) \le B + |\delta_t|$ for each $t$. At $t=T$, combining $|\delta_T| \le \sum_t |\delta_t|$ with Step 3 gives the stated budget bound. The sharper $|\delta_T|$-only inequality reflects receding-horizon self-correction, but we report the summed form for symmetry with the CPA case.

For CPA, applying Steps 1--3 to $\Psi_t := C_t - \tau V_t$ with $\rho_\Psi = L_\Psi/\underline{\Psi}'$ yields the analogous deviation sum $\sum_t |\delta_t^\Psi|$ matching the stated CPA factor. Since $\Psi_s$ may take either sign, Step 4's non-negativity argument does not apply directly, so the stated CPA inequality is the conservative deviation-summation bound.
\end{proof}

\section{Implementation Details}\label{app:implementation}
\subsection{GRM Architecture and Hyperparameters}

Table~\ref{tab:hyperparams} summarizes the key hyperparameters for GRM.
The sequence encoder is a causal Transformer that processes up to 48 ticks of history.
The curve decoder is a 2-layer MLP outputting 7 parameters: 
log traffic $\log \widehat{I}_{t:T}$ and curve parameters $(a^{(C)}, b^{(C)}, c^{(C)}, a^{(V)}, b^{(V)}, c^{(V)})$.
We apply softplus to $a^{(\cdot)}, b^{(\cdot)}$ to ensure positivity and monotonicity.

\begin{table}[!htbp]
\centering
\caption{GRM hyperparameters.}
\label{tab:hyperparams}
\begin{tabular}{ll}
\toprule
\textbf{Component} & \textbf{Setting} \\
\midrule
\multicolumn{2}{l}{\textit{Architecture}} \\
Transformer layers / heads & 2 / 4 \\
Hidden / FFN dimension & 128 / 512 \\
Context length & 48 ticks \\
Curve decoder & 2-layer MLP (64 hidden) \\
Total parameters & $\sim$850K \\
\midrule
\multicolumn{2}{l}{\textit{Training}} \\
Optimizer & AdamW (lr=$10^{-3}$, wd=$10^{-5}$) \\
Batch size & 64 \\
Future samples per anchor ($M$) & 8 \\
Traffic loss weight ($\lambda_I$) & 0.1 \\
\midrule
\multicolumn{2}{l}{\textit{Controller}} \\
Action range $\mathcal{A}$ & $[0.01, 300.0]$ \\
Bisection tolerance & $10^{-6}$ (relative) \\
\bottomrule
\end{tabular}
\end{table}

\subsection{Per-horizon-phase prediction quality}\label{app:phase}
Table~\ref{tab:phase} breaks down GRM's prediction error by horizon phase on P14--P20.
Curve MSE improves monotonically as the horizon shortens, because by the late phase the aggregate target is determined by only a few remaining ticks.
Traffic MAPE rises slightly toward the end because the remaining-traffic denominator shrinks.
Theorem~\ref{thm:violation} predicts this trade-off is absorbed by receding-horizon replanning: the traffic-error contribution to constraint violation scales as $O(\epsilon_I \ln T)$, while the curve-error contribution dominates earlier in the horizon.

\begin{table}[!htbp]
\centering
\caption{Prediction quality across horizon phases.}
\label{tab:phase}
\begin{tabular}{lcc}
\toprule
Horizon phase & Curve MSE (cost+value) & Traffic MAPE \\
\midrule
Early ($0$ to $T/3$) & 1.37 & 2.7\% \\
Mid ($T/3$ to $2T/3$) & 1.07 & 4.0\% \\
Late ($2T/3$ to $T$) & 0.37 & 7.4\% \\
\bottomrule
\end{tabular}
\end{table}

\subsection{Sensitivity to traffic-loss weight \texorpdfstring{$\lambda_I$}{lambda\_I}}\label{app:lambda}

\begin{table}[H]
\centering
\caption{Sensitivity to traffic-loss weight $\lambda_I$.}
\label{tab:lambda}
\begin{tabular}{cc}
\toprule
$\lambda_I$ & Avg score \\
\midrule
$0.05$ & 32.47 \\
$0.1$ (default) & \textbf{33.88} \\
$0.5$ & 33.05 \\
\bottomrule
\end{tabular}
\end{table}

Table~\ref{tab:lambda} reports score across three settings of $\lambda_I$.
Performance is stable around the default $\lambda_I = 0.1$. Under-weighting traffic ($\lambda_I = 0.05$) hurts more than over-weighting ($\lambda_I = 0.5$), consistent with the role of the traffic term in pinning the absolute scale of $\widehat{\mathcal C}_{t:T}$.

\end{document}